\documentclass[letterpaper, 10 pt, conference]{ieeeconf}  

\IEEEoverridecommandlockouts                              

\usepackage[utf8]{inputenc}
\usepackage{amsmath}
\usepackage{graphicx}
\usepackage{caption}
\usepackage{subcaption}
\usepackage{amsfonts}
\usepackage{xcolor}
\usepackage{multicol}
\usepackage{amsmath,amssymb,amsfonts}
\usepackage{subfig}
\newtheorem{theorem}{Theorem}
\newtheorem{proposition}{Proposition}
\newtheorem{lemma}{Lemma}

\newtheorem{definition}{Definition}

\IEEEoverridecommandlockouts 

\title{\bf Building Hybrid B-Spline And Neural Network Operators
}
\author{Raffaele Romagnoli, Jasmine Ratchford, and Mark H. Klein
\thanks{R. Romagnoli is with the Department of Electrical and Computer Engineering, Carnegie Mellon University, Pittsburgh, PA 15213, USA:
        {\tt\small rromagno@andrew.cmu.edu}}
\thanks{J. Ratchford and M. H. Klein are with the Software Engineering Institute, Carnegie Mellon university, Pittsburgh, PA 15235, USA:
        {\tt\small  jratchford@sei.cmu.edu }
        {\tt\small  mk@sei.cmu.edu }}
}

\begin{document}

\maketitle
\begin{abstract}
Control systems are indispensable for ensuring the safety of cyber-physical systems (CPS), spanning various domains such as automobiles, airplanes, and missiles. Safeguarding CPS necessitates runtime methodologies that continuously monitor safety-critical conditions and respond in a verifiably safe manner. A fundamental aspect of many safety approaches involves predicting the future behavior of systems. However, achieving this requires accurate models that can operate in real time.
Motivated by DeepONets, we propose a novel strategy that combines the inductive bias of B-splines with data-driven neural networks to facilitate real-time predictions of CPS behavior. We introduce our hybrid B-spline neural operator, establishing its capability as a universal approximator and providing rigorous bounds on the approximation error. These findings are applicable to a broad class of nonlinear autonomous systems and are validated through experimentation on a controlled 6-degree-of-freedom (DOF) quadrotor with a 12 dimensional state space.
Furthermore, we conduct a comparative analysis of different network architectures, specifically fully connected networks (FCNN) and recurrent neural networks (RNN), to elucidate the practical utility and trade-offs associated with each architecture in real-world scenarios.
\end{abstract}

\section{Introduction}
Control systems play a vital role in ensuring the safety of cyber-physical systems (CPS), including automobiles, airplanes, and missiles \cite{platzer2018logical}. During design, mathematical techniques are employed to prevent these systems from entering unsafe states. However, due to the dynamic interaction with the physical environment, it is impossible to foresee every potential safety hazard. Therefore, development-time methodologies often need to be supplemented with runtime approaches \cite{de2019mixed} that continuously monitor safety-critical conditions and react in a verifiably safe manner.

An approach to ensuring safety involves predicting the system's future behavior. However, this approach requires that model of behavior is accruate and can be executed in real time. This is exceedingly challenging for complex systems. The notable success of machine learning (ML) techniques has drawn researchers' attention towards scientific ML (SciML) methods, which can learn from governing equations and data \cite{baker2019workshop}\cite{lu2021learning} and provide real-time solutions. Scientific ML has been recently used to control applications for system identification, model predictive control (MPC), analysis and verification, and digital twins \cite{nghiem2023physics}\cite{piquad_2022}. 




We are inspired by DeepONets \cite{lu2021learning}, a neural network architecture strategy that can be used to learn generalized non-linear operators. For an autonomous system, these spaces could be considered the initial conditions, and the to a unique solution (trajectory) that transports the autonomous system to a desired end state. Although there are several different types of neural operators, a primary advantage of DeepONets is their architecture, which is primarily composed of separate neural networks that learn basis functions and coefficients, enables the variation which we study in this paper.


A primary feature of DeepONets is that they learn basis functions for representing the an approximate solution. This is an advantage for complex systems where the basis functions are not easily described. However, it introduces two places where the stochastic nature of statistical models introduce uncertainty: the trunk net where the basis functions are learned and in the branch net where the coefficients of the basis functions are learned. In contrast, estimating trajectories from initial conditions, as might be useful for autonomous systems safety analysis, have established, well-researched, methods for describing continuous bases that impose an inductive bias on the data-based learning - splines. In this paper we examine solutions that combine neural operator theory with a learned basis function that reflects a strong inductive bias. Specifically, we use B-splines as the basis functions \cite{de1978practical}.

The key parameters that need to be learned when using B-splines to approximate continuous functions are control points. We train a neural network to learn the mapping from a set of differential equations' initial conditions to the control points of a B-spline approximation. The general approach of using B-spline functions to approximate the solutions of PDEs in the field of Finite Element Methods (FEMs) is referred to as isogeometric analysis (IGA). A similar solution has been proposed in \cite{doleglo2022deep} where a neural network generates the B-spline coefficients for a parametrized family of vector fields to reduce computational costs.

This paper contributes by analyzing this approach within the framework of deep neural operators, demonstrating its property as a universal approximator, and providing bounds on the approximation error. These results are developed for a general nonlinear autonomous system, and the approach is tested on a controlled 6-degree-of-freedom (DOF) quadrotor with a state space of 12 dimensions. Additionally, a comparison between different network architectures, namely fully connected neural networks (FCNN) and recurrent neural networks (RNN), is provided.

A challenge in using ML-based approaches is proving bounds on the uncertainties. The method proposed in this paper aims to open a new direction towards this goal, which is fundamental in all aspects of CPSs, particularly for safety. We distinguish our focus research research that focuses on incorporating physics \cite{piquad_2022} for improved predictive performance and other physics-based guarantees. In light of this, the choice of B-splines is motivated by their interesting properties, such as the convex hull property, where the predicted solution belongs to the convex hull generated by the control points. This property can be beneficial for safety inference, as it only requires checking a finite number of control points. These aspects represent the next steps we aim to address in the near future along the generalization to nonautonomous nonlinear systems.

\section{Preliminaries}
This section is divided in the following parts: nonlinear systems \cite{Khalil:1173048}, neural networks \cite{bittanti1996identification}, the universal approximation theorem for nonlinear operators \cite{chen1995universal},  B-splines functions \cite{jetto2017b} and finally the problem statement.
\subsection{Nonlinear Systems}\label{subsec_A}
Consider a nonlinear system expressed in the state-space form:
\begin{equation}\label{eqn:nl_sys}
    \dot{x} = f(t, x)
\end{equation}
where \( x \in \mathbb{R}^n \) is the state vector with \( n \) components, and the function \( f : \mathbb{R} \times \mathbb{R}^n \rightarrow \mathbb{R}^n \) is Lipschitz on a ball around the initial condition \( x_0 \in \mathcal{B}_r = \left\lbrace x \in \mathbb{R}^n : \lVert x - x_0 \rVert \leq r \right\rbrace \) and for the time interval \( [t_0, t_1] \): \(\lVert f(x) - f(y) \rVert \leq L \lVert x - y \rVert\) for all \( x, y \in \mathcal{B}_r \). $\Vert \cdot \Vert$ is the Euclidean norm. From the local Cauchy's theorem, there exists a unique solution on the time interval \( [t_0, t_0 + \delta] \) with a positive \( \delta \in \mathbb{R} \). According to Peano's theorem, the solution can be written as:
\begin{equation} \label{VolterraInt}
    x(t) = x_0 + \int_0^t f(s, x(s)) ds
\end{equation}
which is a continuous function \( x : \mathbb{R} \rightarrow \mathbb{R}^n \). Let us consider the space of continuous functions over the time interval \( [a, b] \), defined as \( \mathcal{X} = C([a,b]; \mathbb{R}^n) \), which is a Banach space with norm \( \lVert x \rVert_C = \max_{t\in [0, \delta]} \lVert x(t) \rVert \). Considering \(\mathcal{S} = \left\lbrace x \in \mathcal{X} : \lVert x - x_0 \rVert_C \leq r \right\rbrace\) it is possible to show that the right-hand side of \eqref{VolterraInt} is a mapping \( \mathcal{P} : \mathcal{S} \rightarrow \mathcal{S} \), and moreover, this mapping is a contraction:
\begin{equation}\label{mapping_P}
    x(t) = \mathcal{P}(x)(t)
\end{equation}
where the function \( x(t) \) is a fixed point.

We assume that the solution can be extended over the time interval \( [0, T] \) and the Lipschitz condition holds for a compact set \( W \subset \mathbb{R}^n \). Equation \eqref{eqn:nl_sys} is the general representation of a non-autonomous system, where the explicit dependence on time \( t \) can be due to an external input \( u \in \mathbb{R}^p \):
\[
    \dot{x} = f(x, u)
\]
If the input signal is a state feedback control law of the form \( u = \gamma(x) \), where \( \gamma : \mathbb{R}^n \rightarrow \mathbb{R}^n \), the resulting closed-loop system can be represented as an autonomous system of the form 
\begin{equation}\label{eqn:aut_system}
    \dot{x} = f(x)
\end{equation}
Without loss of generality, we assume that $x = {0}$ is an equilibrium point (i.e., $ f({0} ) = {0}$) for \eqref{eqn:aut_system}, hence the origin is contained in \( W \).

\subsection{Neural Networks}
The fundamental element of a neural network (NN) is the standard neuron, which takes $n$ inputs $\{ u_1, u_2, \ldots, u_n \}$, where $u_i \in \mathbb{R}$, and transforms them into a scalar value $y \in \mathbb{R}$. The activation signal for the $i$-th neuron is computed as:
\begin{equation}\label{activation_fun}
    s = u_1\theta_{i,1} + u_2\theta_{i,2} + \cdots + u_n \theta_{i,n} + \theta_{i,0}
\end{equation}
where $\theta_{i,j} \in \mathbb{R}$ with $j=1,\ldots,n$ are the weights that need to be estimated, and $\theta_{i,0}$ is called the bias. The activation signal \eqref{activation_fun} serves as the input to the activation function $\sigma: \mathbb{R} \rightarrow  \mathbb{R}$, which is typically nonlinear (e.g., sigmoidal, ReLU, etc.). Assuming we have a single-layer network with $N$ parallel neurons that use the same activation function, the output can be expressed as:
\begin{equation}\label{scalar_out}
    \begin{split}
        y &= \sum_{i=1}^N w_i \sigma\left(\sum_{j=1}^n \theta_{i,j}u_j+\theta_{i,0}\right) \\
        &= \sum_{i=1}^N w_i \sigma\left(\theta_i^T u + \theta_{i,0}\right)
    \end{split}
\end{equation}
where $\theta_{i}=[\theta_{i,1},\ldots,\theta_{i,n}]^T$, $u=[u_1,\ldots,u_n]^T$. If the output layer has a dimension greater than one, \eqref{scalar_out} describes the $m$-th component of the output vector:
\begin{equation}\label{k_component}
    y_m =  \sum_{i=1}^N w_{i,m} \sigma\left({\theta_{i,m}}^T u + \theta_{i,0}^{m}\right)
\end{equation}
In the case of multi-layer networks known as feedforward neural networks (FNNs), the output of one layer becomes the input of the next one:
\begin{equation}\label{multi_layer}
    y^{[l+1]}_m = \sum_{i=1}^N w_{i,m} \sigma\left({\theta_{i,m}^{[l+1]}}^T y^{[l]} + \theta_{i,0}^{m^{[l+1]}}\right)
\end{equation}
where $[l+1]$ indicates the $(l+1)$-th layer, and $y^{[l]}$ is the corresponding input vector, which is the output of the previous layer. Equations \eqref{k_component} and \eqref{multi_layer} illustrate how a single-layer network with scalar output can be extended to handle vector outputs and multi-layer FNNs. Consequently, the theory of universal approximation of neural networks primarily focuses on the case described in \eqref{scalar_out}. In the following section, we provide basic definitions and results on the universal approximation of nonlinear operators with neural networks.

\subsection{The Universal Approximation Theorem for Nonlinear Operators}
By considering a specific class of activation functions, Tauber-Wiener (TW) functions, we present three main theorems from \cite{chen1995universal}. Theorems 1 and 2 deal with the approximation of functions and functionals with NNs, respectively. Finally, Theorem 3 combines the two previous results to approximate nonlinear continuous operators.

\begin{definition}
A function $\sigma: \mathbb{R} \rightarrow \mathbb{R}$ is called a Tauber-Wiener (TW) function if all linear combinations $\sum_{i=1}^N c_i\sigma(\lambda_i z+\theta_i)$, $\lambda_i, \theta_i, c_i \in \mathbb{R}$, $i=1,2,...,N$ are dense in every $C([a,b];\mathbb{R})$.
\end{definition}

Let $g: K \rightarrow \mathbb{R}$, where $K \subset \mathbb{R}$ is a compact set, $\mathcal{U} \subset C(K;\mathbb{R})$ is a compact set, and $\sigma \in (TW)$ is an activation function.

\begin{theorem}
For any $\epsilon > 0$, there exists a positive integer $N$, $\theta_0^i \in \mathbb{R}$, $\theta_i \in \mathbb{R}^n$, $i=1,...,N$ independent of $g \in C(K;\mathbb{R})$, and constants $c_i(g)$, $i=1,...,N$ depending on $g$ such that
\begin{equation}
    \left\vert g(z) - \sum_{i=1}^N c_i(g)\sigma(\theta_i^Tz + \theta_0^i) \right\vert < \epsilon
\end{equation}
holds for all $z \in K$ and $g \in \mathcal{U}$. Moreover, each $c_i(g)$ is a linear continuous functional defined on $\mathcal{U}$.
\end{theorem}

The last statement is due to the fact that in the proof, $c_j(g)$ represents the Fourier coefficients of $g$. If $g: \mathbb{R} \rightarrow \mathbb{R}$ and $K=[a, b]$, it can be associated with a specific time interval. Note that $\mathbb{R}^n$ with $n \geq 1 \in \mathbb{N}$ is a Banach space.

Now, let $\mathcal{X}$ be a Banach space, $\mathcal{K} \subseteq \mathcal{X}$, $\mathcal{V}$ a compact set in $C(\mathcal{K};\mathbb{R})$, and $g$ a continuous functional defined on $\mathcal{V}$.

\begin{theorem}
For any $\epsilon > 0$, there exist a positive integer $N$, $m$ points $z_1,...,z_m \in \mathcal{K}$, and real constants $c_j$, $\zeta_0^j$, $\xi_{i,j}$, $i=1,...,N$, $j=1,...,m$ such that
\begin{equation}
    \left\vert g(u) - \sum_{i=1}^N c_i \sigma\left(\sum_{j=1}^m \xi_{i,j}u(z_j) + \zeta_{0}^j\right) \right\vert < \epsilon
\end{equation}
holds for all $u \in \mathcal{V}$.
\end{theorem}

Now, let $\mathcal{K}_1 \subseteq \mathcal{X}$, $K_2 \subseteq \mathbb{R}^n$ be two compact sets, $\mathcal{V}$ a compact set in $C(\mathcal{K}_1;\mathbb{R})$, and $\mathcal{G}$ a nonlinear continuous operator mapping $\mathcal{V}$ into $C(K_2;\mathbb{R})$.

\begin{theorem}
For any $\epsilon > 0$, there exist positive integers $M$, $N$, $m$, constants $c_i^k$, $\zeta_0^{jk}$, $\xi_{i,j}^k \in \mathbb{R}$, points $\theta_k \in \mathbb{R}^n$, $z_j\in \mathcal{K}_1$, $i=1,...,M$, $k=1,...,N$, $j=1,...,m$ such that
\begin{equation}
   \begin{split}
    &\left\vert \mathcal{G}(u)(y) - \sum_{k=1}^N \sum_{i=1}^M c_i^k \sigma\left( \sum_{j=1}^m \xi_{i,j}^k u(z_j) + \zeta_0^{jk} \right) \right. \\
    &\hspace{4.5cm}\left. \cdot \sigma(\theta_k^Ty + \theta_0^k) \right. \Bigg\vert < \epsilon
\end{split}
\end{equation}
holds for all $u \in \mathcal{V}$ and $y \in K_2$.
\end{theorem}

\subsection{B-splines}

B-splines are piece-wise polynomial functions derived from slight adjustments of Bezier curves, aimed at obtaining polynomial curves that tie together smoothly. In this work, we are interested not in representing geometric curves, but functions as in \cite{jetto2017b}. For this reason, we consider a parameter $t \in K =[a,b] \subseteq \mathbb{R}$, and $(c_i)_{i=1}^{\ell} \in \mathbb{R}$ as a set of $\ell$ control points for a spline curve $s(t)$ of degree $d$, with non-decreasing knots $(\hat{t}_i)_{i=1}^{\ell+d+1}$.

\begin{equation}\label{scalarbspline}
s(t) = \sum_{i=1}^{\ell} c_i B_{i,d}(t) \quad \text{for} \quad t \in K\subseteq \mathbb{R},
\end{equation}

where $B_{i,d}(t)$, $d > 1$, is given by the Cox-de Boor recursion formula \cite{de1978practical}:

\begin{equation}
B_{i,d}(t) = \frac{t - \hat{t}_i}{\hat{t}_{i+d} - \hat{t}_i} B_{i,d-1}(t) + \frac{\hat{t}_{i+d+1} - t}{\hat{t}_{i+d+1} - \hat{t}_{i+1}} B_{i+1,d-1}(t),
\end{equation}

and

\begin{equation}
B_{i,0}(t) = \begin{cases}
1, & \hat{t}_i \leq t < \hat{t}_{i+1}, \\
0, & \text{otherwise}.
\end{cases}
\end{equation}

\textit{Property 1:} Any value assumed by $s(t)$, $\forall t \in K$, lies in the convex hull of its $\ell + 1$ control points $(c_i)_{i=1}^{\ell}$.

\textit{Property 2:} Suppose that the number $\hat{t}_{i+1}$ occurs $m$ times among the knots $(\hat{t}_j)_{j=i-d}^{m+d}$ with $m$ an integer bounded by $1 \leq m \leq d + 1$, e.g., $\hat{t}_i < \hat{t}_{i+1} = \cdots = \hat{t}_{i+m} < \hat{t}_{i+m+1}$, then the spline function $s(t)$ has a continuous derivative up to order $d - m$ at knot $\hat{t}_{i+1}$.

This property implies that the smoothness of the spline can be adjusted using multiple knot points. A common choice is to set $m = d + 1$ multiple knot points for the initial and final knot points. This way, Equation \eqref{scalarbspline} assumes the first and final control points as initial and final values.

\textit{Property 3:} the B-spline basis functions are continuous in $t$, $B_{i,d}(t) \in C(K;\mathbb{R})$, and bounded \cite{prautzsch2002bezier}.

By defining the vectors 
\begin{equation}\label{compact_cB}
\begin{split}
    c &\triangleq [c_{1} \, c_{2} \, \dots \, c_{{\ell}}]^T\\
    B_d(t) &\triangleq [B_{1,d}(t) \, B_{2,d}(t) \, \dots \, B_{\ell,d}(t)]
\end{split}
\end{equation}
\eqref{scalarbspline} can be rewritten as \(s(t)=B_d(t)c\).

B-splines are generally used to represent curves. However, in our problem, we require a function of time that can be conceptualized as a 2D curve. Therefore, each control point should be a 2D point. To ensure that $s(t)$ is indeed a function of time, we associate the parameter $t$ with the time variable that varies within the interval $[a, b]$. We enforce that the position of each control point along the axis representing time $t$ remains fixed by partitioning the interval $[a, b]$ into $\ell - 1$ equispaced sub-intervals. By doing so, we obtain the formulation presented in \eqref{compact_cB}, where each control point is associated with only a scalar coefficient. 

For the purpose of this paper we are interested to the B-spline representation  for the approximation of the solution of \eqref{eqn:aut_system} with $x\in \mathbb{R}^n$ consists of a set of $n$ scalar B-spline, one for each component of the state vector. In general this B-spline is represented as
\begin{equation}\label{Bs_n}
    \underline{s}(t) = \underline{B}_d(t) \underline{c}
\end{equation}
where $\underline{s}:\mathbb{R}\rightarrow \mathbb{R}^n$, and
\begin{eqnarray}
\underline{B}_d(t) &\triangleq& \mathrm{diag}[B_d(t)]\\
\underline{c}&\triangleq& \left[\underline{c}_1^T, \cdots, \underline{c}_n^T\right]^T\\
\underline{c}_i &\triangleq& \left[c_{i1},\cdots,c_{i\ell}\right]^T.
\end{eqnarray}
The index $i$ indicates the $i$-th components of the multidimensional B-spline $\underline{s}$.


\subsection{Problem Statement}
In this work, our focus is on providing future state estimation of nonlinear closed-loop systems of the form \eqref{eqn:aut_system} by approximating the operator $\mathcal{P}(x)(t)$ \eqref{mapping_P} via a deep operator network that combines deep neural networks and B-spline functions. By momentarily assuming that \eqref{eqn:aut_system} is scalar, our goal is to approximate $\mathcal{P}(x)(t)$ as follows:
\begin{equation}\label{P_approx}
x(t) = \mathcal{P}(x)(t) \approx \hat{\mathcal{P}}(x)(t) = \sum_{i=1}^\ell c_i(x_0)B_{i,d}(t)
\end{equation}
for $t\in[0,T]$, where the coefficients $c_i(x_0)$ represent the B-spline's control points along the time interval $[0,T]$, determining the shape of $x(t)$ depending on the initial condition $x_0$. 
The $c_i(x_0)$ is a functional as depending of the initial condition $x_0$ which is approximated by a deep neural network. Building upon the theoretical framework established in \cite{chen1995universal}, we extend the theory to encompass \eqref{P_approx}. Subsequently, we delve into the multidimensional scenario, where $x\in \mathbb{R}^n$, particularly when addressing error bounds.

\section{B-spline-Based Deep Neural Operator}

Let us consider the same conditions of Theorem 3, modified for the problem of finding an approximating operator \eqref{P_approx} for the scalar case. We consider $\mathcal{X} = C([0,T];\mathbb{R})$ as a Banach space with the norm $\Vert x \Vert_C = \max_{t\in[0,T]}\Vert x(t)\Vert$. Let $\mathcal{S}\subseteq \mathcal{X}$ be a compact set, and the operator $\mathcal{P}(x)(t)$ maps an element of $\mathcal{S}$ into itself. The scalar initial condition $x_0 = x(0) \in \mathcal{S}$.

\begin{theorem}
For any $\epsilon > 0$, there exist positive integers $M$, $\ell$, constants $c_i^k$, $\zeta_0^{jk}$, $\xi_{i}^k \in \mathbb{R}$, functions $B_{k,d}(t)\in C([0,T],\mathbb{R})$, and $x_0 \in \mathcal{S}$ such that
\begin{equation}\label{new_bs_op}
    \left\vert \mathcal{P}(x)(t) - \sum_{k=1}^\ell \sum_{i=1}^M c_i^k \sigma \left(\xi_{i}^k x_0 + \xi_0^{ik}\right) \cdot B_{k,d}(t) \right\vert \leq \epsilon
\end{equation}
holds for all $x_0 \in \mathcal{S}$ and $t \in K_2$.
\end{theorem}

\begin{proof}
To prove this theorem, we follow a similar approach used in \cite{lu2021learning}. In this case, the branch network $\sum_{i=1}^M c_i^k \sigma\left(\xi_{i}^k x_0 + \xi_0^{ik}\right)$ is expressed in the same form as in Theorem 3, except the trunk network is replaced by the B-spline basis functions $B_{k,d}(t)$. Recalling the proof of Theorem 3 (Theorem 5 in \cite{chen1995universal}), we need to show that we can write 
\begin{equation} \label{proof_cond}
    \left\vert \mathcal{P}(x)(t) - \sum_{k=1}^\ell c_k(\mathcal{P}(x))B_{k,d}(t)\right\vert \leq \epsilon/2    
\end{equation}
Theorem 1 cannot be used since $B_{k,d}(t)$ is a polynomial function. Instead, we can use the Weierstrass approximation theorem \cite{Davidson2009RealAA}. This theorem is proven by using a linear combination of Bernstein polynomials used as basis functions. This linear combination of polynomials is also used to generate B{\'e}zier curves which are a particular subtype of B-spline functions \cite{chudy2023linear}. In particular, it is possible to transform a B{\'e}zier curve into a B-spline and vice-versa \cite{prautzsch2002bezier}, therefore we can switch from a representation with B-spline basis function $B_{k,d}(t)$ to Bernstein polynomials and vice-versa. This implies that we can use B-splines as universal approximators of functions, and for this reason, \eqref{proof_cond} holds. 
\end{proof}

For the case of $x \in \mathbb{R}^n$, it is possible to redefine Theorem 4 as Theorem 2 in \cite{lu2021learning}. Specifically, we need to define one B-spline for each component of the state, and the input to the network is a vector $x_0\in \mathbb{R}^n$, where for each neuron $i$, we have $\sigma\left(\sum_{j=1}^n \xi_{i,j}^k x_0^j+ \zeta_0^{ik} \right)$. 

\subsection{Initial condition and control points mapping}
Let us assume that for a specific initial condition $x_0\in W \subseteq \mathbb{R}$ we compute the particular solution $x(t)$ for $t\in [0,T]$. We can take $N$ samples of the solution $x(t)$ by properly choosing a sampling step $h$ and we can form a vector of measurements $y_m=[x(0), x(h), x(2h), ..., x((N-1)h)]^T$. From \eqref{compact_cB}, we can write
\begin{equation} \label{ls_problem}
    y_m = C_m c_{x_0}; \qquad C_m=\left[\begin{array}{c}
         B_d(0)  \\
         B_d(h) \\
         \vdots\\
         B_d((N-1)h)
    \end{array}\right] \in \mathbb{R}^{N \times \ell}. 
\end{equation}
The vector of control points defined in \eqref{compact_cB} is renamed here as $c_{x_0}$ to indicate the solution to a specific trajectory $x(t)$ generated from $x_0$. From the B-spline properties, the rank of $C_m$ is equal to $\ell$, then it is possible to find the vector of control points $c_{x_0}$ by using the pseudo inverse $c_{x_0}=(C_m^TC_m)^{-1}C_m^Ty_m$ which minimizes $\Vert y_m - C_m c_{x_0} \Vert^2$. As a consequence of the Weierstrass approximation theorem, it follows that by appropriately choosing values for $\ell$ and $h$, the error can be reduced to a desired level.
\[
\vert x(t) - B_d(t) c_{x_0} \vert < \epsilon.
\]
For each initial condition $x_0 \in W$ we can find a unique sequence of control point since the initial control point corresponds with the initial condition itself. This means that there is a mapping $\mathcal{M}$ that associates to each initial condition $x_0\in W$ a sequence of $\ell$ control points. 
From Theorem 4 it follows that 
\begin{equation}
\mathcal{M}(x_0) \approx \hat{\mathcal{M}}(x_0) = \left[\begin{array}{c}
     \sum_{i=1}^M c_i^1 \sigma(\xi_i^1x_0+\zeta_0^{i1})  \\
     \vdots\\
     \sum_{i=1}^M c_i^{\ell} \sigma(\xi_i^{\ell}x_0+\zeta_0^{i\ell})
\end{array} \right]    
\end{equation}
where the neural network is used to approximate the mapping $\mathcal{M}$. In this way an infinite dimensional problem has been reduced to a finite one. The use of B-splines guarantees that the approximation of the solution of $x(t)$ is a continuous function defined for any $t\in [0, T]$. In the next subsection we define the problem for the multidimensional case $x\in \mathbb{R}^n$, formalize the existence of the mapping $\mathcal{M}$, and provide some error bounds.
\subsection{Multidimensional Solution and Error Bounds}
For the general problem where $\hat{\mathcal{P}}(x)$ approximates the solution of $n$ dimensional ODEs such as \eqref{eqn:aut_system},  \eqref{P_approx} can be written as
\begin{equation}
    \hat{\mathcal{P}}(x)(t) = \underline{B}_d(t)\hat{\mathcal{M}}(x_0)
\end{equation}
where $\underline{B}_d(t) \in \mathbb{R}^{n \times (n\cdot \ell)}$ accordingly with \eqref{Bs_n}, and the network $\hat{\mathcal{M}}: \mathbb{R}^n \rightarrow \mathbb{R}^{n\cdot \ell}$ has $n$ inputs and $n \cdot \ell$ outputs. 
\begin{lemma}
    For the autonomous system described by \eqref{eqn:aut_system}, where $x_0$ denotes the initial condition in $W \subset \mathbb{R}^n$, and $x(t)$ denotes a solution for the time interval $[0, T]$, with $h > 0$ representing a sampling time and $\ell > 0$ indicating the number of B-spline control points for each component of the state $x$, there exists a mapping $\mathcal{M}: \mathbb{R}^n \rightarrow \mathbb{R}^{n \cdot \ell}$ such that
    \begin{equation}\label{ls_bound}
        \Vert x(t) - \underline{B}_d(t)\mathcal{M}(x_0)\Vert \leq \epsilon
    \end{equation}
    where, $\underline{B}_d(t) \in \mathbb{R}^{n \times (n \cdot \ell)}$.
\end{lemma}
\begin{proof}
    For each $i$-th component of a solution $x(t)$ obtained for a specific initial condition $x_0 \in W$, we can formulate the least squares problem \eqref{ls_problem}, where $y_{m,i}=C_{m}c_{i,x_{0}}$. Here, we define:
    \[
    \underline{y}_m = \left[y_{m,1}^T, \cdots, y_{m,n}^T \right], \; \underline{C}_m =\mathrm{diag}[C_m]
    \]
    and $\underline{c}_{x_0}=\left[{c}_{1,x_0}^T,\cdots, {c}_{n,x_0}^T \right]^T$. This problem assumes the least squares form $\underline{y}_m=\underline{C}_m\underline{c}_{x_0}$, which can be solved using the pseudo-inverse. As $\underline{C}_m$ is a block diagonal matrix where each block has $\ell$ independent columns, a unique sequence of control points is generated by the least squares. Consequently, there exists a mapping $\mathcal{M}$ that maps any $x_0$ into the B-spline control points domain, ensuring \eqref{ls_bound} holds.
\end{proof}
Similarly to the approach in \cite{qin2021data}, where a FNN is employed to approximate the one-step evolution function, we utilize a generalization of Theorem 1 in \cite{pinkus1999approximation} to establish the following proposition:

\begin{proposition}
Considering the conditions of Lemma 1 and a mapping $\mathcal{M}$ satisfying \eqref{ls_bound}, there exists an FNN $\hat{\mathcal{M}}$ such that 
\begin{equation}
    \Vert \mathcal{M}(\cdot) - \hat{\mathcal{M}}(\cdot)\Vert \leq \gamma
\end{equation}
for any $x_0 \in W$.
\end{proposition}

We define $\hat{x}(t) \triangleq \underline{B}_d(t) \hat{\mathcal{M}}(x_0)$ and $\tilde{x}(t) \triangleq \underline{B}_d(t) {\mathcal{M}}(x_0)$, where $\mathcal{M}$ and $\hat{\mathcal{M}}$ retain their meanings from Proposition 1.

\begin{lemma}
Under the conditions of Lemma 1, for the problem of finding $\hat{\mathcal{P}}(x)$, there exists $M>0$ such that
\begin{equation}
    \Vert x(t) - \hat{x}(t)\Vert \leq \epsilon + M\gamma
\end{equation}
for any $t\in [0,T]$, where $\epsilon$ and $\gamma$ are the bounds related to the least squares and FNN approximation, respectively.
\end{lemma}

\begin{proof}
For any $t\in [0,T]$, we can write
\begin{equation}
    \begin{split}
        \Vert x(t) - \hat{x}(t) \Vert &= \Vert x(t) - \hat{x}(t) + \tilde{x}(t) - \tilde{x}(t) \Vert\\
        &\leq \Vert x(t) - \tilde{x}(t) \Vert + \Vert \tilde{x}(t) - \hat{x}(t)\Vert
    \end{split}  
\end{equation}
From \eqref{ls_bound} of Lemma 1, we have $\Vert x(t) - \tilde{x}(t) \Vert\leq \epsilon$, whereas
\begin{equation}
    \begin{split}
        \Vert \tilde{x}(t) - \hat{x}(t)\Vert&=\Vert \underline{B}_d(t)(\mathcal{M}(x_0)-\hat{\mathcal{M}}(x_0))\Vert\\
        &\leq \Vert \underline{B}_d(t)\Vert \Vert\mathcal{M}(x_0)-\hat{\mathcal{M}}(x_0) \Vert \leq M\gamma
    \end{split}
\end{equation}
where $\gamma$ results from Proposition 1, and $M$ results from the boundedness of the B-spline basis functions \cite{prautzsch2002bezier}.
\end{proof}

\section{Experiments}
We demonstrate the value of this technique through several experiments with a quadrotor with nonlinear dynamics and controlled with a linear quadratic regulator (LQR) controller \cite{romagnoli2023software}. The input in every experiment was the 12-dimensional initial condition, consisting of the 3-D position, velocity, angular orientation, and angular velocity. The output is a set of control points given a known sampling interval and time horizon. 

We generated data by randomly choosing initial conditions within a 12-D ball limited by 2 meters in each spatial dimension, 2 meters per second in each velocity dimension, $\pi/4$ radians in each angular dimension and 5 radians per second in each angular velocity dimension. All generated trajectories assumed a desired equilibrium point at $\vec{0}$ and we refer to it as $x_{eq}$. For each trajectory, the controls points that define each trajectory were calculated using least squares fitting and a 3rd order B-spline ($d=3$). For all experiments we used 5000 initial conditions for training. To augment our dataset and encourage rotational equivariance, we rotated our randomly generated initial conditions by $\pi$, $-\pi/2$, $-\pi/4$, $-\pi/6$, $\pi/4$,$\pi/3$, $\pi/2$ radians around the Z axis. 

In our first experiment, we set the number of control points to 50, using a time horizon of 2.5 seconds. We used a simple fully-connected neural network (FCNN) with 12 layers. Each hidden layer has a width of 120 neurons and the input and output layers are dictated by the space dimension (12) and number of control points (50). Therefore the output layer has a width of 600. We optimized over the mean squared error between the predicted sampled trajectory and the trajectory found through least squares fitting. In order to examine model variance, training was reinitialized 10 times and training was stopped after 2 hours or when the loss had reached about $10^{-5}$,which ever condition was satisfied first. One model out of the 10 attempts did not converge within this loss condition and, thus, was not used in the subsequent analysis.

A FCNN does not exploit the physical properties of the system. In our second experiment we combined a gated recurrent unit (GRU) in combination with a FCNN in a recurrent neural network (RNN) architecture. Both the GRU and the subsequent fully connected layers had a width of 120. The fully connected layers had a depth of 3. This resulted in a lower validation loss (mean squared error) for fewer than half as many parameters, as can be seen in Table \ref{tab:time_comparisons}. 

\section{Results}
We used the trajectories and control points found using least squares fitting via singular value decomposition with Python and Numpy as as ground truth in training the neural networks. In order to evaluate and compare across our experiments, we evaluated against a test set containing 1000 new initial conditions. In Table \ref{tab:time_comparisons} we compare the timing between different experiments and the ODE Solver We find that the FCNNs is 5 times faster than least squares fitting and the RNN is 5 times slower than SciPy ODE solver. The increase in evaluation time between the FCNN and the RNN is a consequence of a for-loop required to calculate the entire trajectory. Notably, the ODE solver used is a python wrapper around a ODEPACK library whereas the neural networks were built using PyTorch. Therefore, we advise caution when comparing these evaluation times.

\begin{table*}
\centering
\caption{Time comparisons for different methods.}
\begin{tabular}{|r|l|l|l|l|l|l|l}\hline
Method & T[sec] & $\ell$ &   \# Parameters & Mean Final MSE (Loss)& \# Models & Evaluation Time [sec]\\ \hline
ODE solver & 2.5 & 50 & - & - &  - & $0.0049\pm 0.0015$ \\
FCNN & 2.5 & 50  & 219360 & $1.4 \pm 0.2 \times 10^{-5}$& 9 & $0.00097 \pm 0.00023$ \\
RNN & 2.5 & 50 & 78732 & $2.9 \pm 0.6 \times 10^{-6}$& 10 &$0.028\pm 0.003$ \\\hline
\end{tabular}
\label{tab:time_comparisons}
\end{table*}

In both cases, the error is significantly correlated with the initial condition, as parameterized by the distance in the 12-D ball of the initial condition. This can be seen in Figure \ref{fig:r_vs_e}. This result is aligned with several factors: 1) Since $x_{eq}$ is asymptotically stable, the radius of the initial condition in the unit ball serves as a measure of the distance the quadrotor is from it; 2) Our initial conditions were randomly sampled within the cube of the radius to balance comprehensive sampling throughout the volume and sampling at different radii. Consequently, the density of training points decreases with the radius.


\begin{figure}
\begin{subfigure}[t]{0.45\textwidth}
\includegraphics[width=3.3in]{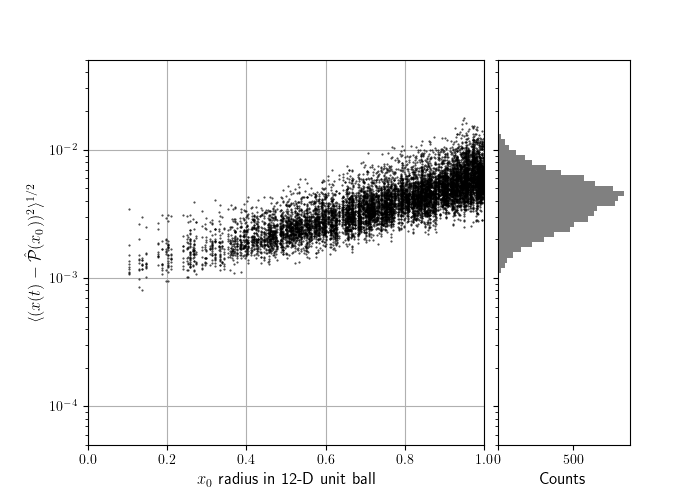}
\end{subfigure}

\begin{subfigure}[t]{0.45\textwidth}
\includegraphics[width=3.3in]{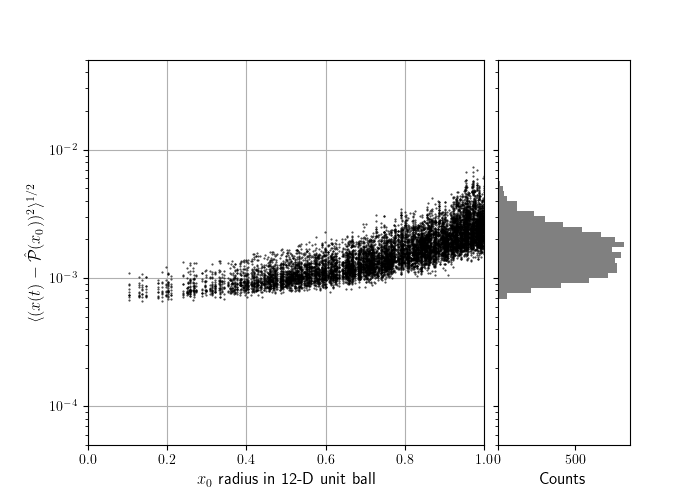}
\end{subfigure}
\caption{ (top) FCNN and (bottom) RNN root mean squared error versus initial condition radius. The error is positively correlated with distance.The overall error can be seen in the profile histogram right of the scatter plot.}

\label{fig:r_vs_e}
\end{figure}

Because summary statistics can only provide a partial understanding of our prediction accuracy, we chose four specific initial conditions to analyze qualitatively by varying the initial condition in a single dimension. These are shown in Figure \ref{fig:ic_assess}. The initial conditions we considered were (with all unspecified values set to zero):
\begin{itemize}
\item  $x^{y}_{0} = (0.8 \, m, 1.2\,m)$ - Red (solid and dotted) lines indicate the least-squares ("ground truth") trajectories and the pink shading indicates the range of the respective neural-network approximated trajectories.
\item   $x^{y}_{0} = 1.0 \,m , x^{dy}_{0} = (0.8 \;m/s, 1.2\, m/s)$ - Blue lines and shaded regions.
\item   $x^{y}_{0} = -1.0 \,m , x^{d\theta}_{0} = (4 \,rad/s, 6 \; rad/s)$ - Green lines and shaded regions.

\item   $x^{y}_{0} = -1.0 \,m ,\; x^{d\theta_y}_{0} =5 \, rad/s, x^{\theta_y}_{0} =\pi/16\cdot(3, 5)$ - Orange lines and shaded regions.
\item $x_{0} = R/2$ where R is the maximum value in any dimension, except for $x_z = -R/2$. Black lines and gray shaded regions. In this instance, all directions were varied slightly.

\end{itemize}

\begin{figure*}[t!]
\centering
\begin{subfigure}[t]{0.48\textwidth}
    \centering
    \includegraphics[scale=0.4,trim={0 2cm 0 3.5cm},clip]{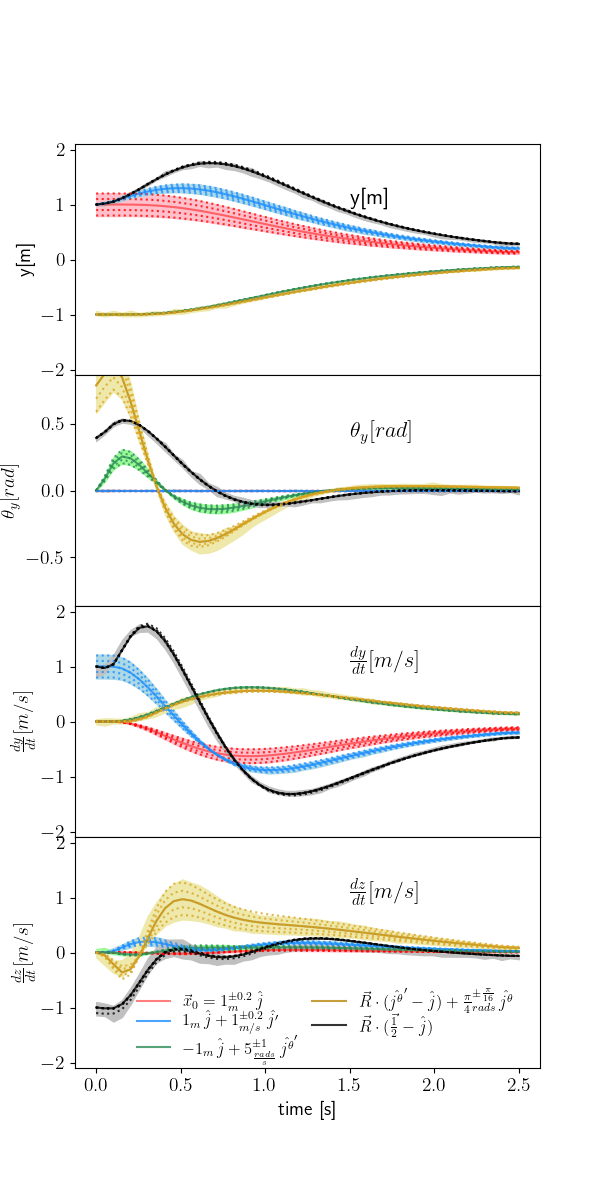}
\end{subfigure}\begin{subfigure}[t]{0.48\textwidth}
    \centering
    \includegraphics[scale=0.4,trim={0 2cm  0 3.5cm},clip]{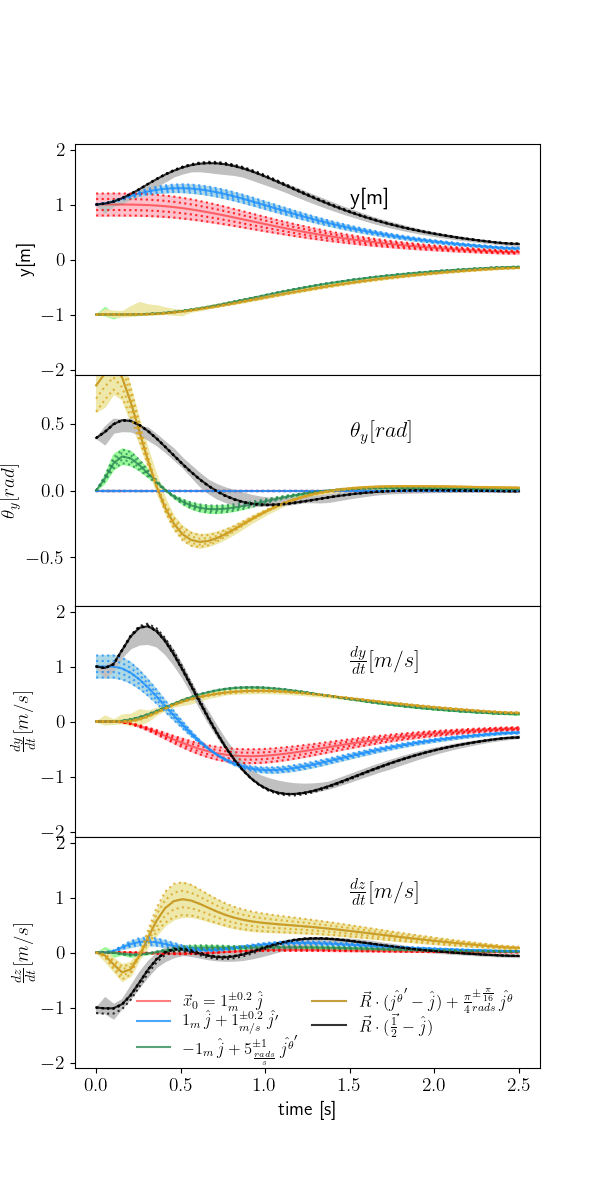}
\end{subfigure} 
\caption{ (left) FCNN and (right) RNN predicted and least-squares fitted trajectories for sets of initial condition, varied about a central value as indicated on the legend. Only four dimensions are shown to conserve space.}
    \label{fig:ic_assess}
\end{figure*}

This provides anecdotal evidence of the per-trajectory fit and shows how variation due to the neural network architecture and training compare to variation in the initial values. The variation due to initial condition variation is present in both the neural network approximations (shaded regions) and least squares fit of spline control points to the solution found with an ODE solver (lines). However, the shaded regions depicting model variations have uncertainty contributions from both initial conditions \textit{and} model stochasticity. Notably, for both the RNN and the FCNN, the variance in initial conditions dominates except for the final case where all we dimensions are significantly non-zero. This agrees with the earlier assessment that the primary indicator of uncertainty is the proximity of the initial condition to the edge of the ball of potential initial conditions. Furthermore, excepting the same case, the least squares trajectory is always within the trajectories indicated by the neural network. Note that only four of twelve dimensions are shown for brevity.

If the resulting neural operator (neural network plus b-spline system) reflects the invariances and symmetries of the actual system, there are stronger guarantees that the neural operator will generalize to new, unseen conditions. In Figure \ref{fig:rotation}, we show the difference root mean squared difference between our original testing dataset and the test dataset rotated to 85 evenly spaced angles between $\pi/8$ and $15\pi/8$ and the resulting trajectory, rotated back. An ideal equivariant operator has the property $R_{\theta}^T\mathcal{P}(R_{\theta} x_{0})  = R_{\theta}^T R_{\theta}\mathcal{P}(x_{0}) = \mathcal{P}(x_{0})$
and the root mean squared difference, $\langle(R_{\theta}^T \mathcal{P}(R_{\theta}x_{0})- \mathcal{P}(x_{0}))^2\rangle^{1/2} =0 $. We find that this deviation is comparable to the overall error exhibited by the respective neural networks. Notably, the RNN has a long low-RMSE tail. We suggests that this indicates an recurrent or residual architecture is more well adapted to this problem and we may be able to improve accuracy with a refined recurrent or residual network architecture. 

\begin{figure}[t!]

\centering
\begin{subfigure}[t]{0.5\textwidth}
\includegraphics[width=3.5in]{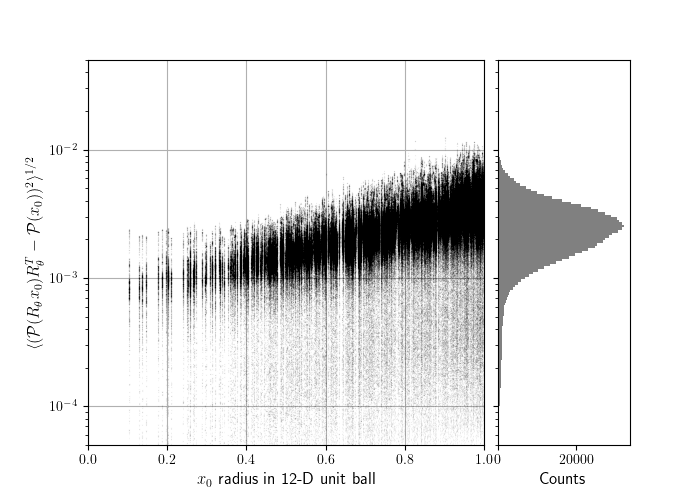}
\end{subfigure}

\begin{subfigure}[t]{0.5\textwidth}
\includegraphics[width=3.5in]{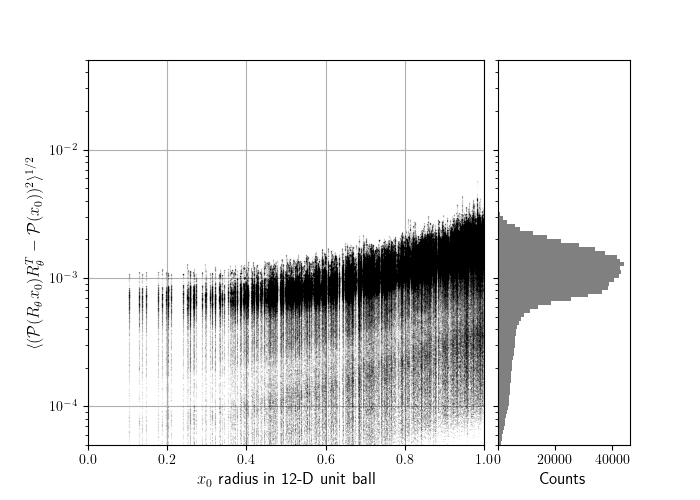}
\label{fig:rnn_rotation_r_vs_e}
\end{subfigure}

\caption{The initial condition is rotated around the Z axis to 85 different angles from $\pi/8$ to $15\pi/8$. The (top) FCNN or (bottom) RNN is applied. The root mean squared difference is compared to the radius of the initial condition in the 12-D ball.}
\label{fig:fcnn_rotation r_vs_e}
\label{fig:rotation}
\end{figure}
\section{Conclusions}
We are exploring various improvements and adjustments to the architecture and training of the neural network that could enhance performance. Notably, neural architectures like residual blocks naturally capture the iterative nature of sequential control points. Additionally, incorporating a long-short-term memory cell (LSTM) could enhance temporal smoothness. Furthermore, embedding the rotational symmetries of the quadrotor into the neural network design using equivariant network structures for dynamical systems \cite{villar_scalars_2021, wang_incorporating_2021}, such as group equivariant convolutional neural networks or SympNets \cite{pmlr-v119-romero20a, jin2020sympnets}, may also improve performance.

In addition to enhancing the contributions presented in this paper, our next endeavor involves developing a framework to estimate and fine-tune the neural network errors by leveraging the B-spline convex hull property. Ensuring that the actual state trajectory resides within the convex hull of the control points enables real-time evaluations of safety violations. Furthermore, a crucial aspect of our future work entails extending this framework to non-autonomous systems, which is pivotal for advancing the development of neural network-based controllers.


\section{acknowledgements}
Copyright 2024 Carnegie Mellon University and Raffaele Romagnoli

This material is based upon work funded and supported by the Department of Defense under Contract No. FA8702-15-D-0002 with Carnegie Mellon University for the operation of the Software Engineering Institute, a federally funded research and development center.

[DISTRIBUTION STATEMENT A] This material has been approved for public release and unlimited distribution.  Please see Copyright notice for non-US Government use and distribution.
DM24-0319

\bibliographystyle{ieeetr}
\bibliography{biblio.bib}
\end{document}